\documentclass{article}

\usepackage[round]{natbib}

\usepackage[preprint]{neurips_2025}

\usepackage[utf8]{inputenc} 
\usepackage[T1]{fontenc}    
\usepackage[colorlinks=true, linkcolor=black, citecolor=black, urlcolor=blue]{hyperref}
\usepackage{url}            
\usepackage{booktabs}       
\usepackage{amsfonts}       
\usepackage{nicefrac}       
\usepackage{microtype}      
\usepackage{xcolor}         

\usepackage{wrapfig}

\usepackage{algorithmic}
\usepackage{algorithm}

\usepackage{amsmath}
\usepackage{graphicx}
\usepackage{amssymb}
\usepackage{amsthm}
\usepackage{caption}
\usepackage{subcaption}
\usepackage{bbm}
\usepackage{mathtools}
\usepackage{courier}
\usepackage{algorithm}
\graphicspath{{./figures/}}

\usepackage{chngcntr}
\usepackage{apptools}

\usepackage{makecell}

\numberwithin{equation}{section}
\newcommand{\ee}{{\rm e}\hspace{1pt}}

\newcommand{\dd}{\hspace{1pt}{\rm d}\hspace{0.5pt}}

\newcommand{\veps}{\varepsilon}

\newtheorem{thm}{Theorem}

\newtheorem{defn}[thm]{Definition}

\DeclareMathOperator*{\argmax}{argmax}

\title{Differentially Private In-Context Learning with  Nearest Neighbor Search}

\author{%
  Antti~Koskela \\
  Nokia Bell Labs \\
  \And
  Tejas Kulkarni \\
  Nokia Bell Labs \\
  \AND
  Laith Zumot \\
  Nokia \\
}

\begin{document}

\maketitle

\begin{abstract}

Differentially private in-context learning (DP-ICL) has recently become an active research topic due to the inherent privacy risks of in-context learning. However, existing approaches overlook a critical component of modern large language model (LLM) pipelines: the similarity search used to retrieve relevant context data.
In this work, we introduce a DP framework for in-context learning that integrates nearest neighbor search of relevant examples in a privacy-aware manner. Our method outperforms existing baselines by a substantial margin across all evaluated benchmarks, achieving more favorable privacy-utility trade-offs. To achieve this, we employ nearest neighbor retrieval from a database of context data, combined with a privacy filter that tracks the cumulative privacy cost of selected samples to ensure adherence to a central differential privacy budget.
Experimental results on text classification and document question answering show a clear advantage of the proposed method over existing baselines.

\end{abstract}

\section{Introduction} 

In-context learning (ICL)~\citep{iclpaper} is a popular way to tailor a generic language model’s response to a specific context/domain. A typical ICL pipeline involves first preparing a guiding prompt that contains several task related examples, such as question-answer pairs, and then asking the language model to generate a response for the query, conditioned on the examples provided. A key feature of ICL is that it does not involve compute heavy operations of updating model weights and typically API or prompt-only access to LLM is sufficient.

Privacy risks in LLMs due to memorization are well known~\citep{memorization1,memorization2,memorization3}. One line of research deals with leakage in fine-tuning~\citep{dpftllm2,dpftllm1} or pretraining~\citep{dpptllm}.
Another line of research attempts to recover training records using clever prompt engineering~\citep{attack_pe1,attack_pe2,attack_pe3}.
Specifically for ICL,~\citet{mia_icl1,mia_icl2,mia_icl3} have proposed membership inference attacks to detect the membership of a test data point in a private prompt.

Differentially Private In-Context Learning (DP-ICL) is an active area of research, currently being explored along two parallel directions:

\begin{itemize}

    \item DP Synthetic example generators~\citep{tangprivacy,gao2025data,amin2024private}: These methods generate synthetic examples token-by-token by privately releasing the mean logits from several partitions sensitive examples. The next token with the largest weight is selected as the next token. The generated examples can be used as demonstrations in multiple downstream ICL tasks without incurring additional privacy costs. While one-time privacy costs is an attractive feature, these methods are computationally expensive and rely on logit outputs from the LLMs, which may not be easily available in many scenarios. Moreover, experiments in these works are limited to simpler tasks such text classification and information extraction.

    \item Pay-per-use~\citep{wuprivacy}: The private set of examples is partitioned into $k$ shards, each of which is associated with a given prompt. The model generates one response for each of the $k$ shards. For text classification, the final output is released via private voting using the shard responses. For text generation problems, the private responses are aggregated in either keyword or embedding space and released privately. LLM is then asked to provide the final response based on the top keywords or mean embeddings. While the privacy cost scales with the number of test queries answered, the method is easy to parallelize, does not require access to logits and has the capacity to generate very high-quality responses in a wide range of tasks. Our work improves on this method.

\end{itemize}

A related line of work, often referred to as private prediction, studies how to obtain differentially private predictions from non-private models~\citep{dwork2018privacy, papernot2018pate, bassily2018model, zhu2020private, zhu2023private}. Methods in this class typically perturb model outputs or the voting scores, and some of them also use $k$-nearest neighbor (kNN) search~\citep{zhu2020private, zhu2023private}.
Interestingly, both the DP Synthetic Example Generators the Pay-per-use methods can be viewed as instances of this broader private prediction framework: they provide privacy guarantees only at the prediction stage, while reusing non-private models. However, kNN methods have not yet been incorporated into the DP-ICL setting, which has unique characteristics due to the compositional and prompt-based nature of in-context learning.

The methods by~\citet{wuprivacy} use Poisson sampling to select the examples for the shards. While sampling amplifies the privacy protection, it can pick examples unrelated to the test query. It has been well-documented~\citep{lu2022fantastically,icl_bad_examples2,icl_bad_examples3} that the output of ICL is sensitive to the examples used, and randomly sampled examples can lead to increased prediction uncertainty, potentially resulting in worse performance compared to 0-shot predictions. Therefore, example selection has emerged as an important research direction in ICL~\citep{icl_survey}.

We emphasize that the embedding based kNN search of demonstrations is a standard component in information retrieval systems such as those designed for retrieval augmented generation (RAG)~\citep{icl_index} and can be easily plugged into an existing ICL pipeline. Surprisingly, we are unaware of any work on DP-ICL that uses kNN indexing despite their popularity.



\subsection{Our contributions}
\begin{itemize}

    \item  Through use of privacy filters~\citep{feldman2021individual}, we integrate nearest neighbor search into existing DP-ICL framework by~\citet{wuprivacy}. The modified solution composes prompts with $k$-nearest neighbors of each test point instead of randomly sampled examples like in the baseline methods by~\citet{wuprivacy} and~\citet{tangprivacy}.

    \item As a theoretical contribution, we provide a fully adaptive  $\delta$-approximate RDP analysis of so-called individual RDP filters.

    \item We carry out experiments on text classification and question answering on benchmark datasets with LLMs  such as Llama3.3-70B-it and Gemini-1.5-flash-8B.
    Our experiments clearly show that the overall privacy-utility trade-off is drastically improved with kNN.

\end{itemize}

\section{DP-ICL with kNN}

We give the required background on DP and the problem setting of ICL in Appendix Section~\ref{sec:background}. Our method is based on the basic primitives of report-noisy-max with Gaussian noise (RNM-Gaussian) and DP keyword space aggregation (DP-KSA) that are also the building blocks of the baseline methods by~\citet{wuprivacy}. Those methods are described in detail in Appendix Section~\ref{sec:icl}. We next describe how to combine those methods with kNN nearest neighbor search of examples.

\subsection{Retrieval of Most Similar Examples}

Instead of retrieving examples via subsampling, we combine the RNM-Gaussian and DP-KSA mechanisms with retrieval of the $k$ most similar examples from the sensitive dataset $X$. A similar approach is taken by~\citet{zhu2023private} for private prediction, though not in combination with in-context learning. Specifically, for each query $q$, we construct the retrieved set $\mathcal{R}(X)$ by selecting the $k$ elements $x_i \in X$ most similar to $q$ under a chosen similarity metric. The retrieved set is then partitioned into $M$ disjoint batches ${B_1,\ldots,B_M}$ for in-context prompting.

\subsection{Retrieval with Limited Sensitivity}

The challenge with kNN retrieval is how to carry out the privacy accounting. The required tool is given by the individual RDP accounting~\citep{feldman2021individual}.
To this end, we also require from the retrieval function $\mathcal{R}$ that its output can change at most by one element in case we change the dataset $X$ by one element.
More formally, the output of the LLM $\mathcal{A}$ consist of the retrieval and the DP-ICL algorithm. So we can think of it as a composition $\mathcal{A} = \mathcal{M} \circ \mathcal{R}$, where $\mathcal{M}$ is the DP-ICL mechanism (e.g., DP-KSA) that takes as an input the set of batches $\{B_1,\ldots,B_M\}$, and $\mathcal{R}(X)$ is the retrieval algorithm that fetches the batches from the input dataset $X$.

Mathematically, we say that $\mathcal{R}$ is \emph{stable under single-element change} if, whenever $X \simeq X'$, the outputs differ by at most two elements:
$
|\mathcal{R}(X) \setminus \mathcal{R}(X')| + |\mathcal{R}(X') \setminus \mathcal{R}(X)| \le 2.
$
In order to limit the sensitivity of the aggregation happening in the mechanism $\mathcal{A}$, we require this property from the retrieval $\mathcal{R}$.
In this work, we focus on the FLAT index for simplicity, as it performs a full exhaustive search and trivially satisfies the mentioned stability property. Extending our proposed method to approximate indexing like IVF or HNSW is a compelling avenue for future work. For example, DP $k$-means methods~\citep{chang2021locally} could be used to implement IVF search, incurring an additional privacy cost while still meeting the stability requirements of the retrieval.

\subsection{Individual RDP Accounting for DP-ICL with kNN}

The rigorous privacy accounting for DP-ICL with kNN retrieval can be carried out using an individual $(\alpha,\varepsilon)$-RDP privacy filter that keeps track of individual  privacy losses and drops from the analysis the data elements for which the cumulative privacy loss is about to cross the pre-determined budget $\varepsilon_{\max}$~\citep{feldman2021individual}.
To this end, we first give the following definitions.

Define $\mathrm{Sub}(S, x_i)$ as the set of datasets obtained from $S$ by substituting the data element $x_i$ by another data element, i.e.,
$$
\mathrm{Sub}(S, x_i) = \left\{ S' \;\middle|\; S' = (S \setminus \{x_i\}) \cup \{x_i'\},\; x_i' \in \mathcal{X} \right\}
$$

The individual $\delta$-approximate $(\alpha,\varepsilon)$-RDP privacy filter is described in the pseudocode of Algorithm~\ref{alg:individual_filter}. Notice that in each step $t$, the adaptively chosen mechanism $\mathcal{A}_t$ is of the form $\mathcal{A}_t = \mathcal{M} \circ \mathcal{R}_t$, where the retrieval function $\mathcal{R}_t$ depends adaptively on the query $q_t$ chosen at iteration $t$. I.e., the data elements that are used by the DP-ICL mechanism $\mathcal{M}$ at step $t$, depend on the query and the set of data elements that still have their privacy budget left.
When using the DP-KSA algorithm, we fix the iteration-wise failure probability $\delta_i$. However, it could also be chosen adaptively.

\begin{algorithm}
\caption{Adaptive composition $\mathcal{A}^{(T)}$ with Rényi filter}
\begin{algorithmic}[1]
\STATE \textbf{Input:} Dataset $X$, and privacy budget $(\varepsilon_{\max},\delta_{\max})$.
\STATE Set the active set of data elements to be the whole dataset: $S=X$.
\FOR{$t = 1, \ldots, T$}
    \FOR{all data entries $x_i \in S$}
        \STATE Compute individual $\delta_i$-approximate RDP parameters for the chosen mechanism $\mathcal{A}_t$. I.e.,
        \begin{equation*}
        \begin{aligned}
        \varepsilon_t^{(i)} =
        \sup_{S' \in \mathrm{Sub}(S, x_i)} D^{\delta_i}_{\alpha} \left( \mathcal{A}_t(a^{(t-1)}, S) \| \mathcal{A}_t(a^{(t-1)}, S') \right).
        \end{aligned}
        \end{equation*}
    \ENDFOR
    \STATE Update the set $S$ of active data elements:
    $$
    S = \left\{ x_i \, \bigg| \, \sum\nolimits_{j=1}^{t} \varepsilon_j^{(i)} \leq \varepsilon_{\max}, \,\, \sum\nolimits_{j=1}^{t} \delta_i \leq \delta_{\max}\right\}.
    $$
    \STATE Compute $a_t = \mathcal{A}_t(a_{1:t-1}, S)$
\ENDFOR
\STATE \textbf{Return} $(a_1, \ldots, a_T)$
\end{algorithmic}
\label{alg:individual_filter}
\end{algorithm}

The following result is our main theoretical result and is proven in Appendix~\ref{sec:delta_rdp}. It can be seen as a generalization of the RDP filtering result of~\citep[Thm.\;4.5][]{feldman2021individual} and of the
$\delta$-approximate zCDP filtering result of~\citep[Thm.\;1][]{whitehouse2022fully}.

\begin{thm}[Privacy Filter for $\delta$-approximate Rényi Differential Privacy] \label{thm:main}
Let $K \in \mathbb{Z}_+$ define the maximum number of compositions and let $\{\mathcal{M}_i\}_{i=1}^K$ be an adaptively chosen sequence of randomized mechanisms, where each $\mathcal{M}_i$ is $\delta_i$-approximate $(\alpha, \varepsilon_i(\alpha))$-RDP for some $\alpha \geq 1$. Let $\varepsilon_{\max}(\alpha) > 0$ and $\delta_{\max} \geq 0$ define the privacy budgets.
Then, a \emph{privacy filter} that halts when either
$
\sum_{i=1}^{T+1} \varepsilon_i > \varepsilon_{\max}(\alpha)$ or $\sum_{i=1}^{T+1} \delta_i > \delta_{\max}
$
ensures that, the composed mechanism $\mathcal{M}^{(K)} = \big(\mathcal{M}_1, \ldots, \mathcal{M}_K\big)$ is $\delta_{\max}$-approximate $\varepsilon_{\max}(\alpha)$-RDP.

\end{thm}

Similarly, as from the general RDP filters follow results for individual filters~\citep{feldman2021individual}, from the general filtering result of Thm.~\ref{thm:main} it trivially follows that Algorithm~\ref{alg:individual_filter} is
is $\delta_{\max}$-approximate $(\alpha,\varepsilon_{\max})$-RDP. This privacy guarantee can be then converted to a $(\veps,\delta)$-DP guarantee using the conversion formula given in Appendix Eq.~\eqref{eq:conversion}.

\section{Experimental Results}

\subsection{Text Classification}

We first evaluate our kNN-based DP-ICL method on public benchmark text classification datasets AGNews~\citep{ag_news} and TREC~\citep{trec}. For simplicity, we set the privacy parameters such that each sample is used only once. Consistent with~\citet{wuprivacy},  our implementation utilizes 10 shards, each featuring 4 demonstrations. The exact prompt used has been provided in Appendix.

For the nearest neighbor search, we use the “all-MiniLM-L6-v2” model to produce embeddings of unit length with dimension 384. We use FLAT indices for retrievals, which are constructed using FAISS library~\citep{douze2024faiss}.

We carry out the classification experiments using the open-source model OPT-1.3B by Meta that is also available on the Huggingface platform~\citep{huggingface}.
The predictions are generated deterministically.
Figure~\ref{fig:fig_agnews} shows the mean test accuracies for an experiment, where we pick 200 randomly sampled test samples for the AGNews and TREC datasets, respectively, for different values of $\varepsilon$, when $\delta=10^{-5}$. The results for each  $\varepsilon$ are means of 5 independent runs, and the error bars depict 1.96 times the standard deviation, giving the asymptotic 95\% confidence interval.
We exclude the zero-shot results from Figure~\ref{fig:fig_agnews} when using OPT-1.3B since the zero-shot baseline achieves approximately only 58\% test accuracy for AGNews whereas the results for TREC are close to random guessing.
We remark that the test accuracies for the baseline method and for the zero-shot are similar as in~\citep{wuprivacy} that uses the GPT-3 Babbage model which also has approximately 1.3B parameters.
Figure~\ref{fig:fig_agnews} also includes the "nearest neighbor only (dummy)", where RNM-Gauss is run directly on a histogram formed using the counts of the nearest neighbors' labels.

Notice that on the TREC example of Fig.~\ref{fig:fig_agnews}, the performance of DP-ICL with kNN deteriorates as $\varepsilon$ decreases towards 0.5. This can be explained by the fact that as the DP are those of a Gaussian mechanism with sensitivity $\sqrt{2}$, for $\varepsilon=0.5$ the required noise scale is approximately 10 which equals the number of shards (number of votes), which already significantly randomizes the predictions. Naturally, this could be remedied by using more shards.

\begin{figure} [h!]
     \centering
        \includegraphics[width=.48\textwidth]{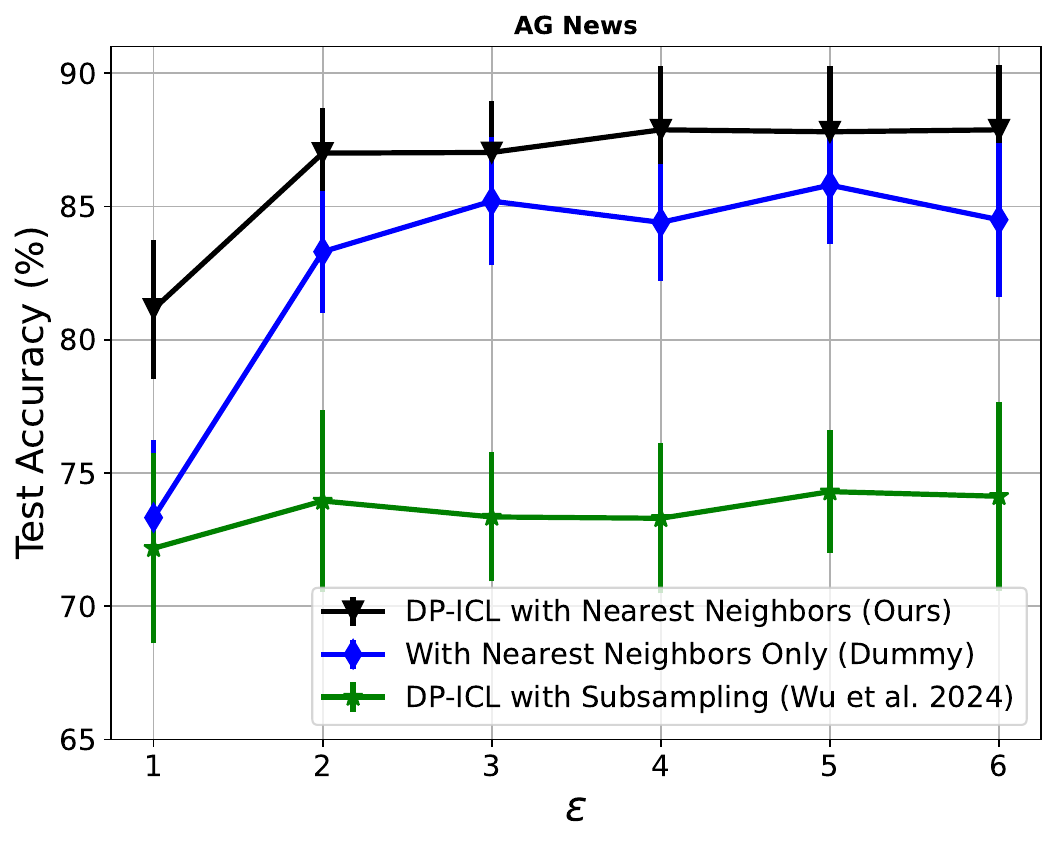}
        \includegraphics[width=.48\textwidth]{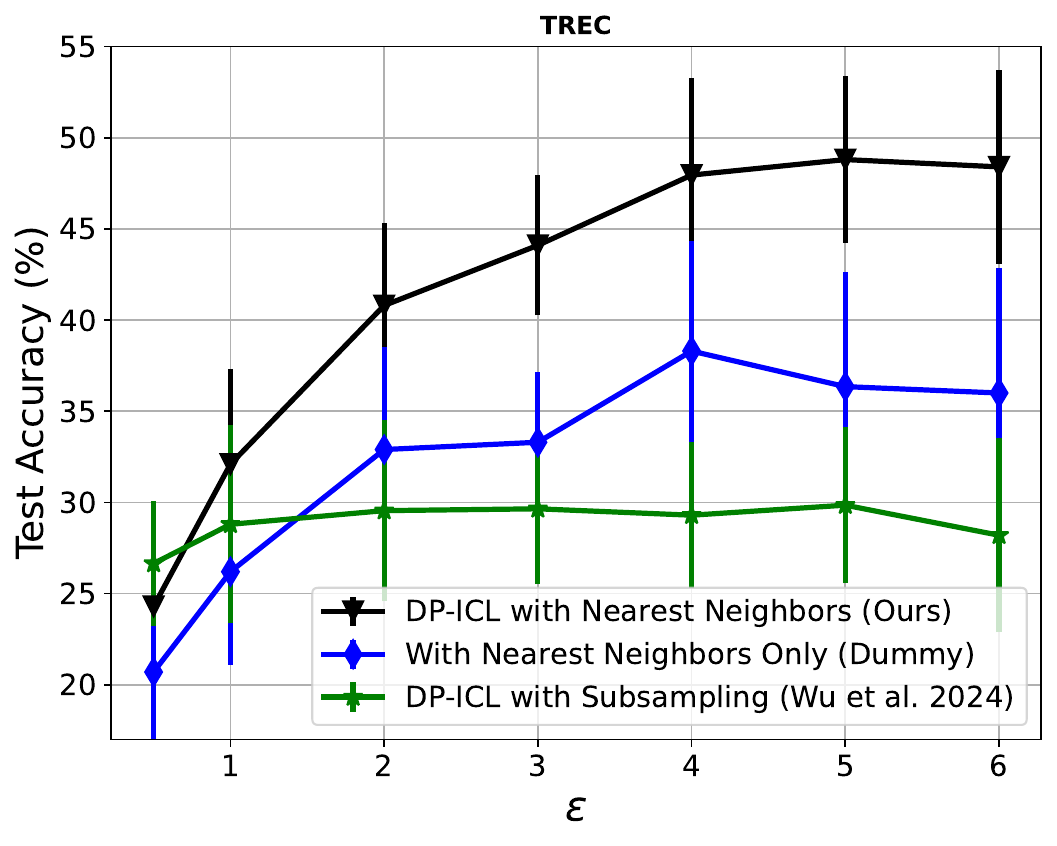}
        \caption{Mean test accuracies for 200 randomly sampled test samples. Left:  AGNews text classification task with 4 classes, averaged over 5  experiments.
        Right: TREC text classification task with 6 classes, averaged over 5 experiments.}
        \label{fig:fig_agnews}
\end{figure}

\subsection{Document Question Answering}
We next compare the methods on the  task of questions answering, on two datasets:

\textbf{Federated version of DocVQA~\citep{tobaben2024neurips}:} This dataset was curated for a competition organized at NeurIPS 23. Each dataset record contains a triplet of the form $\langle$image, question, answer$\rangle$. Each image is a sensitive invoice with confidential details (e.g. payer/payee names, invoice amount, purpose). The original task was to answer multiple questions on each image with limited information leakage. The scope of this work is limited to textual ICL. Therefore, we proxy each image with OCR tokens supplied in the same dataset. We decode and concatenate those tokens to form text sentences ignoring their original position in the image. As concatenated sentences may not form a cohesive paragraph this makes it a challenging dataset.

\textbf{SQuAD v1.1~\citep{rajpurkar2016squad}:} This is a standard reading comprehension dataset, consisting of questions posed on Wikipedia articles. The records are triplets of the form $\langle$paragraph, question, answer$\rangle$.

Both DP-KSA and DP-KSA-kNN  satisfy record-level DP which protects presence of a \emph{single} triplet (document, question, answer). However, both datasets contain multiple questions for each image/paragraph. Therefore, we randomly sample a single question-answer pair for each paragraph and assume that each record belongs to a single user.

 We continue to use  the all-MiniLM-L6-v2 model for embeddings. For DP-KSA-kNN method, we build FLAT index with the text paragraphs using FAISS library.

\begin{table}[h!]
\centering
\begin{tabular}{|l|c|c|}
\hline
\textbf{Dataset} & \textbf{Federated DocVQA} & \textbf{SQuAD} \\
\hline
Demonstration Set & 69,785 & 18,891 \\
Test Query Set & 100 & 100 \\
\hline
\end{tabular}
\caption{Comparison of dataset sizes between Federated DocVQA and SQuAD.}
\label{tab:dataset_sizes}
\vspace{-0.7em}
\end{table}
\textbf{Language models used:} The comparison of distribution of prompt lengths for both datasets is shown in Figure~\ref{fig:prompt_sizes} of Appendix.
We use Llama3.3-70B-it and Gemini-1.5-flash-8B and
 fix the temperature parameter to 0.7 in our API calls for both models.
However, we did not observe much variance in the responses due to ‘to-the-point nature’ of the questions.

\textbf{Accuracy metrics:} Our performance metrics include standard Rouge and Bleu scores. We have described the metrics for completeness in Table~\ref{tab:qa_metrics} in Appendix. All metrics range from 0 to 1. Higher scores imply a higher degree of similarity between two answers.

\textbf{Experimental Results:} Figure~\ref{fig:docvqa_gemini} shows plots for document QA task for 4-shot ICL with shard sizes 10 and 20 for several $\varepsilon$’s.
Plots for the other two (model,dataset) combinations are given in Appendix~\ref{sec:further}.
We use the same randomly sampled 100 test queries for all methods and $\varepsilon$’s. We also include 0-shot responses (obtained without any demonstrations) computed with the same number of shards. Outperforming this baseline is important for any method to justify the use of private demonstrations. Points with $\varepsilon=\infty$ correspond to non-private version of KSA and KSA-kNN.

The main high-level observation across both figures is that most metrics have higher values for the Llama model compared to Gemini. We also note that DP-KSA remains less sensitive to $\varepsilon$’s, whereas,
DP-KSA-kNN improves in many cases specially for high $\varepsilon$’s.
Figure~\ref{fig:docvqa_llama} in Appendix show additional results.

\begin{figure} [ht!]
     \centering
        \includegraphics[width=.490\textwidth]{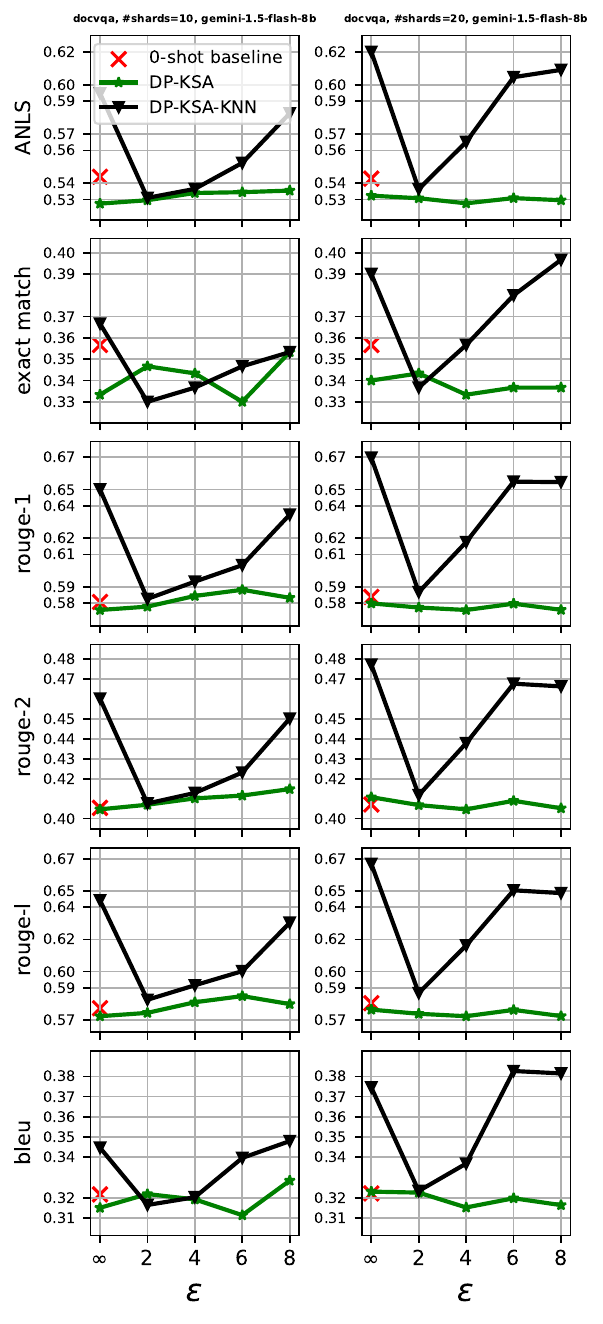}
                 \includegraphics[width=.49\textwidth]{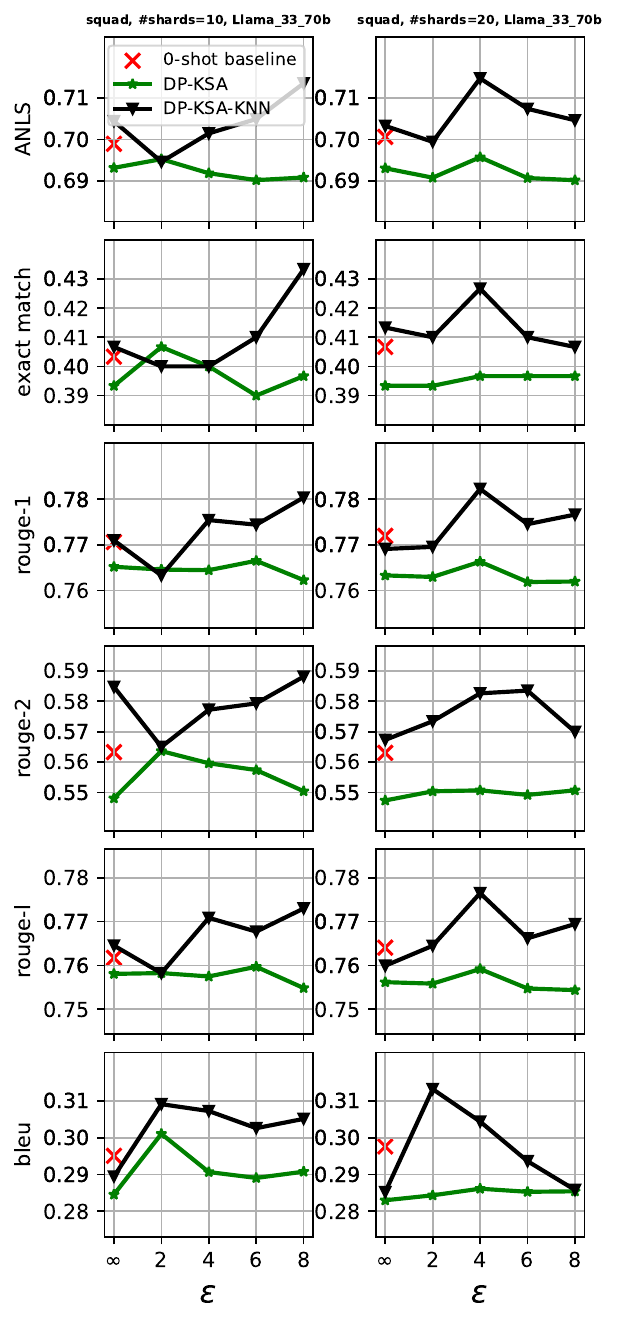}

        \caption{Left: A comparison of DP-KSA and DP-KSA-kNN on a  4-shot Q\&A task on the DocVQA  dataset using Gemini-1.5-flash-8B. Right: A comparison of DP-KSA and DP-KSA-kNN on a 4-shot Q\&A task  on the  SQuAD dataset using the Llama 3.3-70B-It model. The averages are computed over individual metrics for 100 test queries. The higher number indicates a higher degree of similarity between algorithm’s final response and ground truth. We see that the proposed method (DP-KSA-kNN) is superior compared to the baseline  (DP-KSA).
        }
        \label{fig:docvqa_gemini}
\end{figure}

\section{Conclusions}

In this work, we integrate nearest neighbor search based indexing into an existing DP-ICL framework. This is obtained by using the so-called fully adaptive privacy analysis and individual differential privacy filters. Our experiments on private text classification and private question answering tasks show the substantial advantage of our approach. Our method outperforms the zero-shot and  the DP baseline method by~\citet{tangprivacy} though our approach does not benefit from privacy amplification by subsampling.
Interesting research directions in this topic include building DP-ICL solutions utilizing alternative sample indexing and retrieval methods, such as those based on hierarchical clustering like $k$-means or hierarchical navigable small worlds (HNSW).

\bibliography{llm}

\newpage

\appendix

\section{Further Related Literature} \label{sec:literature}


Private in-context learning is typically applied to tasks such as text classification, question answering, and summarization. It often relies on techniques that privatize the counts of output tokens, using either additive noise mechanisms or top-$k$ mechanisms~\citep{tangprivacy,wuprivacy}.  There are data-adaptive refinements of the methdod by~\citet{tangprivacy}, given by~\citet{gao2025data}, and recently also by~\citet{amin2024private}
who use the accurate concentrated differential privacy accounting presented in~\citep{cesar2021bounding} for the exponential mechanism to improve the privacy-utility trade-offs.
Recent paper~\citep{koga2024privacy} considers a DP RAG method, however does not seem to incorporate the ranking of the augmenting samples into the DP mechanism.
Several papers consider a version of exponential mechanism tailored for private top-$k$ selection~\citep{gillenwater2022joint} called jointEM.
Recently, faster version of jointEM has been proposed by~\citet{haofaster}. Another work that is close to our work is that by~\citet{zhu2023private} who also consider individual privacy accounting and kNN similarity search for private prediction although in a way that is not directly applicable to existing DP-ICL methods.
\subsection{Existing Indexing Methods for Similarity Search}

Commonly used nearest neighbor search methods used in ICL and RAGs include

\begin{itemize}

    \item FLAT, Brute-force search where all vectors are stored and compared exhaustively. Suitable for small datasets but not scalable for large-scale search.

    \item IVF (Inverted File Index)
    Partitions the dataset into clusters (Voronoi cells) using the $k$-means algorithm. During a search, only a subset of clusters is probed to reduce computation. Efficient but requires careful tuning of the number of clusters

    \item HNSW (Hierarchical Navigable Small World) Graph-based indexing where points are connected in a proximity graph. Provides fast nearest neighbor search with logarithmic complexity.
\end{itemize}

These are the main methods of the widely used similarity search libraries Faiss (Facebook AI Similarity Search)~\citep{johnson2019billion,douze2024faiss} and Milvus~\citep{2021milvus,2022manu}.
As outlined by~\citet{douze2024faiss}, vector search systems must navigate trade-offs between search accuracy, speed, and memory consumption, which depend heavily on dataset size, vector dimensionality, and the chosen index architecture. Indexing methods like FLAT, IVF and HNSW can be deployed on both CPU and GPU hardware, providing flexibility to optimization in different application contexts. 



\section{Problem Setting and Background} \label{sec:background}

\subsection{In-Context Learning}
We have a private dataset $X =(x_1,x_2,\cdots,x_N)  \in \mathcal{X}^N$ of demonstrations, where $x_i \in [N]$ consists of the content (e.g. an article for classification or  a tuple of text description and a related question for QA task) and possibly some ground truth (e.g. label, answer, text summary).
We also have a prompt only access to a pretrained (autoregressive) language model $\mathrm{LM}$ with a large enough context window. We also have a  function $\mathcal{R}$ that retrieves a subset of $X$ as few-shot examples.
Given a query content $q$ (e.g. news article), we aim to generate the answer tokens $A$ (e.g. class label, answer to a question) $\argmax_A \mathrm{LM}(A | \mathcal{R}(X) + q)$ in a differentially private manner.
The sign ‘+’ denotes the concatenation operation.
Specifically, we want to use $X$ to learn the mapping between $x$’s and $y$’s and improve over 0-shot prediction $\argmax_A \mathrm{LM}(A |  q)$ for an unknown query $q$. We further assume that client and ICL server interact only once for a single query, and $\mathrm{LM}$ does not retain previous interactions with the same client.

\subsection{Differential Privacy}

 We say input sets $X$ and $X'$ are neighbours if we get one by substituting
one element in the other (denoted $X \sim X'$).


A mechanism $\mathcal{M}$
is $(\veps,\delta)$-DP if its outputs are
$(\veps,\delta)$-indistinguishable for neighbouring datasets.

\begin{defn} \label{def:dp}
	Let $\varepsilon \ge 0$ and $\delta \in [0,1]$. 
	Mechanism $\mathcal{M} \, : \, \mathcal{X}^n \rightarrow \mathcal{O}$ is  $(\veps, \delta)$-DP	if for every pair of neighbouring datasets $X,X'$,
	every measurable set $E \subset \mathcal{O}$,
	$$
	\mathbb{P}( \mathcal{M}(X) \in E ) \leq \ee^\varepsilon \mathbb{P} (\mathcal{M}(X') \in E ) + \delta.
	$$
\end{defn}

We will also use the R\'enyi differential privacy (RDP)~\citep{mironov2017} which is defined as follows.
R\'enyi divergence of order $\alpha \in (1,\infty)$ between two distributions $P$ and $Q$ is defined as
$$
D_\alpha(P || Q) = \frac{1}{\alpha - 1} \log \int \left( \frac{P(t)}{Q(t)} \right)^\alpha Q(t) \, \dd t.
$$
By continuity, we have that $\lim_{\alpha \rightarrow 1+} D_\alpha(P || Q) $ equals the KL divergence $KL(P || Q)$.

\begin{defn}

We say that a mechanism $\mathcal{M}$ is $(\alpha,\varepsilon)$-RDP, if for all neighbouring datasets $X,X'$, the output distributions
$\mathcal{M}(X)$ and $\mathcal{M}(X')$ have R\'enyi divergence of order $\alpha$ less than $\varepsilon$, i.e.,
$$
\max_{X \simeq X'} \{ D_\alpha\big( \mathcal{M}(X) || \mathcal{M}(X') \big) , D_\alpha\big( \mathcal{M}(X') || \mathcal{M}(X) \big) \} \leq \varepsilon.
$$
\end{defn}


Certain applications, like the Propose-Test-Release framework we consider, require a relaxation of RDP that allows a small probability of failure. To address this, we consider a $\delta$-approximate version of RDP, which extends the definition to account for a negligible additive failure probability $\delta$.

\begin{defn}
We say a randomized algorithm $\mathcal{M}$ is $\delta$-approximately
$(\alpha, \varepsilon(\alpha))$-RDP with order $\alpha \geq 1$,
if for all neighboring dataset $X, X'$, there exist events
$E$ (depending on $\mathcal{M}(X)$) and $E'$ (depending on $\mathcal{M}(X')$)
such that $\Pr[E] \geq 1 - \delta$ and $\Pr[E'] \geq 1 - \delta$, and we have
$$
    D_{\alpha}(\mathcal{M}(D)|E \, \| \, \mathcal{M}(D')|E')
    \leq \varepsilon.
$$

\end{defn}
We remark that in the application of text classification, we use the common RDP accounting, and for question answering, we need to use the $\delta$-approximate RDP. 

\subsection{$\delta$-Approximate RDP} \label{sec:delta_rdp}

We next review some of the properties of the $\delta$-approximate RDP~\citep[see, e.g.,][]{Bun2016,papernothyperparameter}.

First, recall that a randomized algorithm $\mathcal{M} : \mathcal{X}^n \to \mathcal{Y}$ is $\delta$-approximately $(\alpha, \varepsilon)$-R\'enyi differentially private if, for all neighbouring pairs of inputs $X, X' \in \mathcal{X}^n$, it is $(\alpha, \varepsilon)$-RDP except for a set of measure at most $\delta$. The definition is given more formally as follows.

\begin{defn}
We say a randomized algorithm $\mathcal{M}$ is $\delta$-approximately
$(\alpha, \varepsilon(\alpha))$-RDP with order $\alpha \geq 1$,
if for all neighboring dataset $X, X'$, there exist events
$E$ (depending on $\mathcal{M}(X)$) and $E'$ (depending on $\mathcal{M}(X')$)
such that $\Pr[E] \geq 1 - \delta$ and $\Pr[E'] \geq 1 - \delta$, and we have
$$
    D_{\alpha}(\mathcal{M}(D)|E \, \| \, \mathcal{M}(D')|E')
    \leq \varepsilon.
$$
\end{defn}

If $\mathcal{M}$ is $\delta$-approximate $(\alpha,\veps)$-RDP, we also shortly denote it as
$$
D_\alpha^\delta(\mathcal{M}(x) \| \mathcal{M}(x')) \leq \varepsilon.
$$

Some  basic properties of approximate RDP are as follows~\citep[see, e.g., Appendix E,][]{papernothyperparameter}:

\begin{itemize}
  \item $(\varepsilon, \delta)$-DP is equivalent to $\delta$-approximate $(\infty, \varepsilon)$-RDP.

  \item $(\varepsilon, \delta)$-DP implies $\delta$-approximate $(\alpha, \tfrac{1}{2} \varepsilon^2 \alpha)$-RDP for all $\alpha \in (1, \infty)$.

  \item $\delta$-approximate $(\alpha, \varepsilon)$-RDP implies $(\hat{\varepsilon}, \hat{\delta})$-DP for
  \begin{equation} \label{eq:conversion}
  \hat{\delta} = \delta + \frac{\exp((\alpha - 1)(\hat{\varepsilon} - \varepsilon))}{\alpha} \cdot \left(1 - \frac{1}{\alpha}\right)^{\alpha - 1}.
  \end{equation}

  \item $\delta$-approximate $(\alpha, \varepsilon)$-Rényi differential privacy is closed under postprocessing.

  \item If $\mathcal{M}_1$ is $\delta_1$-approximately $(\alpha, \varepsilon_1)$-Rényi differentially private and $\mathcal{M}_2$ is $\delta_2$-approximately $(\alpha, \varepsilon_2)$-Rényi differentially private, then their composition is $(\delta_1 + \delta_2)$-approximately $(\alpha, \varepsilon_1 + \varepsilon_2)$-RDP.
\end{itemize}

The following is a tailored subsampling amplification result for $\delta$-approximate RDP mechanisms, given by~\citet{wuprivacy}. We need it for evaluating the privacy guarantees of the baseline method.

\begin{thm}[Privacy amplification by Poisson subsampling for approximate RDP,~\citet{wuprivacy}]\label{thm:approx-rdp-sub}

Let $\mathcal{M}$ be a mechanism satisfying $\delta$‑approximate $(\alpha, \varepsilon_M(\alpha))$‑RDP. Let $\mathcal{M}_{\textup{sub}}$ denote the mechanism that applies $\mathcal M$ to a Poisson subsample of the data with sampling probability $\gamma$. Then:
$$
\mathcal{M}_{\textup{sub}} \text{ satisfies } \delta' \text{-approximate } (\alpha,\; \varepsilon_{\textup{sub}}(\alpha))\text{-RDP}
$$
where
$\delta' = \gamma \, \delta$ and
$\varepsilon_{\textup{sub}}(\alpha)$ equals the tightest possible amplification bound for an $\varepsilon_M(\alpha)$-RDP mechanism under Poisson sampling,
with amplification rate adjusted to
$
\frac{\gamma(1-\delta)}{1 - \gamma\delta}.
$
\end{thm}

\section{Baseline DP-ICL Methods} \label{sec:icl}


We next describe the baseline private aggregation methods by~\citet{wuprivacy} upon which our approach builds. They adopt the \emph{Gaussian Report Noisy Max} (RNM), introduced by~\citet{rnm_gauss}, as one of the mechanisms  for privately selecting class labels in classification tasks. For document question answering, where outputs are open-ended and higher dimensional, they operate in a lower-dimensional keyword space, using private mechanisms to identify salient content at the token level. The following two sections describe both methods in detail.

\subsection{RNM‑Gaussian Mechanism for Text Classification}

The RNM‑Gaussian mechanism $M_\sigma$ adds independently sampled Gaussian noise to each bin of the voting histogram $h \in \mathbb{R}^k$ over class labels, where the histogram has global sensitivity $\Delta = \sqrt{2}$ (since a change in one example–query pair affects at most two bins). Specifically:

$$
\widetilde{h}_i = h_i + \mathcal{N}(0, \sigma^2), \quad \text{for all } i = 1, \dots, k,
$$
and the privatized response is obtained via:
$$
\mathcal{M}(h) = \arg\max_i \, \widetilde{h}_i.
$$
When making $T$ private predictions this way, setting
$\sigma = \sqrt{2 T \ln(1.25/\delta)}/\varepsilon$
ensures that the sequence of outputs satisfies $(\varepsilon, \delta)$-differential privacy~\citep{DworkRoth}. More accurate privacy bounds can be obtained via RDP or by using so called privacy profiles~\citep{balle2018gauss}.
In this work, we use RDP for privacy accounting.



\subsection{Keyword Space Aggregation (KSA) for Document Question Answering}

In the document question answering task, the output $A$ of the LLM consists of natural language tokens (e.g., answers or summaries), rather than a fixed class label. To enable private aggregation in this higher-dimensional output space, we adopt the Keyword Space Aggregation (KSA) method. This approach reduces the complexity of the aggregation by projecting responses into a lower-dimensional token space and performing differentially private selection over salient tokens.

Given a query content $q$, a retrieval function $\mathcal{R}$ obtains $M$ disjoint subsets of the private dataset $X$ and construct $M$ in-context prompts. I.e., the retrieved set of batches $\mathcal{R}(X) = \{B_i\}_{i=1}^M$, where each disjoint batch $B_i$ contains a number of data points. For each prompt, the output is sampled from the language model:
$$
O_i(q) := \mathrm{LM}(q \,+\, B_i),
$$
where $O_i(q)$ is the natural language answer generated by the model for the $i$-th prompt. These outputs are then tokenized to form a frequency histogram $h \in \mathbb{R}^D$ over the vocabulary $\mathcal{V}$ of size $D$, where each count $h_t$ corresponds to the number of outputs in which token $t$ appears:
$$
h_t = \left| \left\{\, i : \text{token } t \in O_i(q) \,\right\} \right|, \quad t = 1, \dots, D.
$$

To privately identify the most relevant semantic content, a differentially private mechanism is applied to select the top-$k$ tokens from the histogram $h$. Depending on the vocabulary size $D$, they consider the following approach.



\paragraph{Propose-Test-Release (PTR).} When $D$ is large or unbounded, a PTR mechanism first privately tests whether the frequency gap $h_{(K)} - h_{(K+1)}$ exceeds a threshold. If the test passes, the top-$K$ tokens are released exactly. To determine $K$ privately, one can also perform a noisy argmax:
$$
\arg\max_k \left( h_{(k)} - h_{(k+1)} \right),
$$
where $h_{(k)}$ denotes the $k$-th largest entry in $h$.

The selected keywords $\{t_1, \dots, t_K\}$ are then incorporated into a follow-up prompt that guides the language model to generate a coherent final answer. For example, we use a structured template such as:
\begin{quote}
\texttt{Using the following keywords, answer the question concisely: $t_1$, $t_2$, ..., $t_K$.}
\end{quote}

An alternatively, exponential mechanism (EM) can also be used to release top keywords, however we found the privacy-utility trade-offs of the PTR-based method superior in our experiments compared to EM-based method.

This procedure ensures that the final output reflects the aggregated knowledge across demonstrations, while differential privacy is guaranteed through token-level mechanisms. The KSA method thus enables private, scalable, and semantically meaningful aggregation in the document QA setting.

\subsection{Privacy Amplification via Subsampling}

\citet{wuprivacy} use Poisson subsampling as the retrieval method $\mathcal{R}$ to select demonstration, and to amplify privacy guarantees. For each query $q$, for the retrieved set of batches $\mathcal{R}(X)$, each example $x_i \in X$ is included independently with probability $\gamma$, and partition the sampled set into $M$ disjoint batches for in-context prompting.
This reduces the likelihood of any individual contributing to the final output, leading to improved privacy bounds. In particular, if the aggregation mechanism is $(\varepsilon, \delta)$-DP, then the overall mechanism with subsampling satisfies approximately $(\gamma\varepsilon, \gamma\delta)$-DP, under standard amplification results. The accurate privacy accounting can be carried out either using subsampling results for RDP (Appendix Thm.~\ref{thm:approx-rdp-sub}).

\textbf{RNM-Gaussian with subsampling:} After subsampling and constructing the class histogram $h$, Gaussian noise is added as before:
$
\widetilde{h}_i = h_i + \mathcal{N}(0, \sigma^2),
$
with the final output $\arg\max_i \widetilde{h}_i$ satisfying improved privacy due to subsampling.

\textbf{KSA with subsampling:} In the QA setting, we apply the same subsampling step before generating outputs $O_i(q)$ and aggregating tokens into the histogram $h$. The top-$K$ selection (via PTR) is then performed on the reduced set, benefiting from the same privacy amplification.


Our main contribution is to replace the existing  $\mathcal{R}$ of random subsampling with a kNN-based retrieval of the most relevant examples from the database. While this approach sacrifices the privacy amplification benefits of subsampling, it significantly improves the quality of the generated outputs. As a result, we can tolerate higher noise levels in the aggregation step, ultimately yielding a better overall privacy–utility trade-off.

 The combination of kNNs and individual RDP has also been used in~\citep{zhu2023private} for private classification with kNN search. However, we consider a completely different and a much broader task of ICL. The method~\citep{zhu2023private} cannot be applied to generative tasks such as question answering. Despite kNN’s popularity in non-private ICL/RAG pipelines, no prior work on DP-ICL has considered employing it.


\section{Propose-Test-Release}

In this Section, we give background details on the propose-test-release (PTR) which is also part of the baseline method~\citet{wuprivacy} and which forms also the basis of our DP-KSA-kNN method.

The main idea of DP-KSA implemented with PTR paradigm is that, for the task of releasing the top-$k$ indices of a voting histogram,
if $H(k) - H(k+1) > 2$, then the top-$k$ indices are exactly the same for all neighboring datasets.
Thus, we can release them without additional noise in that case. To ensure this, a DP test of the gap $H(k) - H(k+1)$ has to be carried out.
This whole PTR procedure is depicted in Algorithm~\ref{alg:topkwithptr}.

\begin{algorithm}[H]
\caption{TopKwithPTR}
\label{alg:topkwithptr}
\begin{algorithmic}[1]
\REQUIRE $k$ -- number of top tokens to release; $H$ -- histogram of token counts; $\delta$ -- failure probability
\STATE $d_k \gets H(k) - H(k+1)$
\STATE $\hat{d}_k \gets \max(2, d_k) + \mathcal{N}(0, 4\sigma^2) - \Phi^{-1}(1-\delta; 0, 2\sigma)$
\IF{$\hat{d}_k > 2$}
    \STATE \textbf{return} exact top-$k$ tokens
\ELSE
    \STATE \textbf{return} Terminate (or fallback to zero-shot learning)
\ENDIF
\end{algorithmic}
\label{alg:TopKwithPTR}
\end{algorithm}

The utility can be further optimized, by selecting $k$ that maximizes the gap $H(k) - H(k+1)$ in a privacy-preserving way
using the exponential mechanism. This is depicted in Algorithm~\ref{alg:findbestk}.

\begin{algorithm}[H]
\caption{FindBestK}
\label{alg:findbestk}
\begin{algorithmic}[1]
\REQUIRE $H$ -- histogram of token counts
\FOR{$k = 1$ to $N-1$}
    \STATE $d_k \gets H(k) - H(k+1)$
\ENDFOR
\STATE \textbf{return} $\arg\max_k \left( d_k + r(k) + \text{Gumbel}(4/\varepsilon) \right)$
\end{algorithmic}
\end{algorithm}
For Algorithm~\ref{alg:TopKwithPTR}, we have the following privacy guarantee given in~\citep[Thm.\;11,][]{wuprivacy}.
\begin{thm}[TopKwithPTR Privacy Guarantee]
Let $H$ be the histogram of a set of i.i.d. samples from a bounded-support distribution.
Let $H(k)$ denote the $k$-th largest value in $H$, and suppose we want to release the top-$k$ indices
of $H$ only if $H(k) - H(k+1) > 2$. Define
$$
\hat{d}_k := \max(2, d_k) + \mathcal{N}(0, 4\sigma^2) - \Phi^{-1}(1 - \delta; 0, 2\sigma),
$$
where $\Phi^{-1}(\cdot; 0, 2\sigma)$ is the inverse CDF of a Gaussian distribution with mean $0$
and standard deviation $2\sigma$. Then the mechanism that releases the top-$k$ indices
if $\hat{d}_k > 2$ satisfies $\delta$-approximate $\big(\alpha,\frac{\alpha}{2\sigma^2})$-RDP for all $\alpha \geq 1$.
\end{thm}

The following privacy guarantee is a standard result for the exponential mechanism, based on the "Gumbel max trick" which means that the implementation of the exponential mechanism is equivalent to running report noisy max with properly scaled additive Gumbel distributed noise~\citep[Remark 3.1,][]{DworkRoth}. In practice, when using FindBestK for DP-KSA, we set $k_{\max} = 30$ and $k_{\min}=15$.

\begin{thm}[FindBestK Privacy Guarantee]
Let $d_k := H(k) - H(k+1)$ for $k = 1, \dots, N-1$, and define $r(k)$ to be a regularizer such that
$r(k) = -\infty$ for $k > k_{\max}$ or $k < k_{\min}$, and $r(k) = 0$ otherwise. Then the mechanism
$$
\arg\max_k \left( d_k + r(k) + \text{Gumbel}(4/\varepsilon) \right)
$$
satisfies $\varepsilon$-differential privacy.
\end{thm}

\section{Fully Adaptive $\delta$-Approximate RDP Accounting (Proof of Thm~\ref{thm:main})}

\begin{thm}[Privacy Filter for $\delta$-Approximate Rényi Differential Privacy]
Let $K \in \mathbb{Z}_+$ define the maximum number of compositions and let $\{\mathcal{M}_i\}_{i=1}^K$ be an adaptively chosen sequence of randomized mechanisms, where each $\mathcal{M}_i$ is $\delta_i$-approximate $(\alpha, \varepsilon_i(\alpha))$-RDP for some $\alpha \geq 1$. Let $\varepsilon_{\max}(\alpha) > 0$ and $\delta_{\max} \geq 0$ define the privacy budgets.
Then, a \emph{privacy filter} that halts when either
$$
\sum_{i=1}^{T+1} \varepsilon_i > \varepsilon_{\max}(\alpha) \quad \text{or} \quad \sum_{i=1}^{T+1} \delta_i > \delta_{\max}
$$
ensures that, the composed mechanism $\mathcal{M}^{(K)} = \big(\mathcal{M}_1, \ldots, \mathcal{M}_K\big)$ is $\delta_{\max}$-approximate $\varepsilon_{\max}(\alpha)$-RDP.


\end{thm}

\begin{proof}
We use here the notation used in~\citep{feldman2021individual}.
For $n \in [K]$ and for two neighboring datasets $X$ and $X'$, denote
$$
y^{(n)}=(y_1, \ldots, y_n),
$$
$$
\mathcal{M}^{(n)}(X) = \big( \mathcal{M}_1(X,\veps_1), \mathcal{M}_2(\mathcal{M}_1(X),X,\veps_2), \ldots, \mathcal{M}_n(\mathcal{M}_1(X), \ldots,
\mathcal{M}_{n-1}(X),X,\veps_n) \big),
$$
$$
\mathrm{Loss}^{(n)}(y^{(n)}; X, X', \alpha)  = \left(  \frac{\mathbb{P}( \mathcal{M}^{(n)}(X) = y^{(n)} )  }{\mathbb{P}( \mathcal{M}^{(n)}(X') = y^{(n)})}  \right)^\alpha,
$$
and
$$
\mathrm{Loss}_n(y^{(n)}; X, X', \alpha)  =  \left(  \frac{\mathbb{P}( \mathcal{M}_n(y^{(n-1)},X,\veps_n) = y_n )  }{\mathbb{P}( \mathcal{M}_n(y^{(n-1)},X',\veps_n) = y_n )}  \right)^\alpha.
$$
Since $\veps_n$ depends only on $y^{(n-1)}$ (not directly on the dataset), by the Bayes rule we have that
$$
\mathrm{Loss}^{(n)}(y^{(n)}; X, X', \alpha) = \mathrm{Loss}^{(n-1)}(y^{(n-1)}; X, X', \alpha) \cdot \mathrm{Loss}_n(y^{(n)}; X, X', \alpha).
$$

We next analyze the $\delta$-approximate RDP of the fully adaptive composition, using similar techniques as used in the proof of~\citep[Thm.\;3.1,][]{feldman2021individual}.
In the RDP integrals below, $y^{(K)}$ is distributed according to $ \mathcal{M}^{(K)}(X')$, and by "with probability at least 1-$\delta$" we mean that the given RDP bound holds except with probability at most $\delta$ over the randomness of $ \mathcal{M}^{(K)}(X')$.

Straightforward calculation then shows that
\begin{equation*}
    \begin{aligned}
&\mathbb{E}_{y^{(K)}| \sum_{i=1}^{K} \varepsilon_i \leq \varepsilon_{\max} \, \textrm{and} \, \sum_{i=1}^{K} \delta_i \leq \delta_{\max}} \left(  \frac{\mathbb{P}( \mathcal{M}^{(K)}(X) = y^{(K)} )  }{\mathbb{P}( \mathcal{M}^{(K)}(X') = y^{(K)})}  \right)^\alpha \\
       =&\mathbb{E}_{y^{(K)}} \left[ \mathrm{Loss}^{(K)}(y^{(K)}; X, X', \alpha) \bigg| \sum_{i=1}^K \veps_i \leq \veps_{\max} \textrm{ with probability at least }  1 - \sum_{i=1}^{K} \delta_i  \geq 1 - \delta_{\max}\right] \\
      = &\mathbb{E}_{y^{(K-1)}} \, \mathbb{E}_{y_K} \bigg[ \mathrm{Loss}^{(K-1)}(y^{(K-1)}; X, X', \alpha) \cdot \mathrm{Loss}_K(y^{(K)}; X, X', \alpha) \bigg| \\
            & \quad \quad \quad \veps_K  \leq \veps_{\max} - \sum_{i=1}^{K-1} \veps_i  \textrm{ with probability at least }  1 -  \delta_K \geq 1 - \delta_{\max}+ \sum_{i=1}^{K-1} \delta_i \bigg] \\
    \end{aligned}
\end{equation*}
This implies that
\begin{equation*}
    \begin{aligned}
    &\mathbb{E}_{y^{(K)}| \sum_{i=1}^{K} \varepsilon_i \leq \varepsilon_{\max} \, \textrm{and} \, \sum_{i=1}^{K} \delta_i \leq \delta_{\max}} \left(  \frac{\mathbb{P}( \mathcal{M}^{(K)}(X) = y^{(K)} )  }{\mathbb{P}( \mathcal{M}^{(K)}(X') = y^{(K)})}  \right)^\alpha \\
    \leq &\mathbb{E}_{y^{(K-1)}}  \left[ \mathrm{Loss}^{(K-1)}(y^{(K-1)}; X, X', \alpha) \right]  e^{(\alpha-1)\big( \veps_{\max} - \sum_{i=1}^{K-1} \veps_i\big)}  \\
\end{aligned}
\end{equation*}
with probability at least $1 - \delta_{\max}+ \sum_{i=1}^{K-1} \delta_i$.
Continuing, and using the fact that
\begin{equation*}
    \begin{aligned}
&\mathbb{E}_{y^{(K-1)}}  \left[ \mathrm{Loss}^{(K-1)}(y^{(K-1)}; X, X', \alpha) \right] \\
&= \mathbb{E}_{y^{(K-2)}} \mathbb{E}_{y_{K-1}} \left[ \mathrm{Loss}^{(K-2)}(y^{(K-1)}; X, X', \alpha) \mathrm{Loss}_{K-1}(y^{(K-1)}; X, X', \alpha) \right],
\end{aligned}
\end{equation*}
we have that
\begin{equation*}
    \begin{aligned}
    &\mathbb{E}_{y^{(K)}| \sum_{i=1}^{K} \varepsilon_i \leq \varepsilon_{\max} \, \textrm{and} \, \sum_{i=1}^{K} \delta_i \leq \delta_{\max}} \left(  \frac{\mathbb{P}( \mathcal{M}^{(K)}(X) = y^{(K)} )  }{\mathbb{P}( \mathcal{M}^{(K)}(X') = y^{(K)})}  \right)^\alpha \\
       \leq &\mathbb{E}_{y^{(K-2)}}  \left[ \mathrm{Loss}^{(K-2)}(y^{(K-1)}; X, X', \alpha) \mathrm{Loss}_{K-1}(y^{(K-1)}; X, X', \alpha) \right]  e^{(\alpha-1)\big( \veps_{\max} - \sum_{i=1}^{K-2} \veps_i\big)}  \\
      & \quad \quad \quad \textrm{ with probability at least } \quad  1 - \delta_{\max}+ \sum_{i=1}^{K-2} \delta_i
    \end{aligned}
\end{equation*}
since
$$
\mathbb{E}_{y_{K-1}} \left[\mathrm{Loss}_{K-1}(y^{(K-1)}; X, X', \alpha) \right] \leq e^{(\alpha-1) \veps_{K-1} }
$$
with probability at least $1-\delta_{K-1}$.

Next, continuing integration mechanism by mechanism, we inductively see that with probability at least $1-\delta_{\max}$, we have that
$$
 \mathbb{E}_{y^{(K)}}  \left[ \left(  \frac{\mathbb{P}( \mathcal{M}^{(K)}(X) = y^{(K)} )  }{\mathbb{P}( \mathcal{M}^{(K)}(X') = y^{(K)})}  \right)^\alpha  \right]
\leq e^{(\alpha-1) \veps_{\max}}.
$$
The conditions $\sum_{i=1}^{K+1} \varepsilon_i \leq \varepsilon_{\max}$ and $\sum_{i=1}^{K+1} \delta_i \leq \delta_{\max}$ hold by construction of the filter.
\end{proof}

\section{Example LLM Prompt: Template Used for 4-Shot Text Classification}
\label{sec:classification_prompt}
\begin{verbatim}
Instruction: Classify each article into one of the following categories
separated by comma: class1, class2, .., class_k.
Article: {demo text 1}, Class: {class1} \n
Article: {demo text 2}, Class: {class1} \n
Article: {demo text 3}, Class: {class2} \n
Article: {query text 3}, Class:
\end{verbatim}

\section{Example LLM Prompt: Template Used for 4-Shot QA}
\label{sec:qa_prompt}

For question answering task,we used the following prompt template. The bracketed terms (e.g., \texttt{\{demo text1\}}) indicate placeholders for specific data.
\begin{verbatim}
Read the text: {demo text1}
Answer the question with at most 4 words: {demo question1}
Do not provide a Yes/No answer: {demo answer1}

Read the text: {demo text2}
Answer the question with at most 4 words: {demo question2}
Do not provide a Yes/No answer: {demo Answer2}

Read the text: {demo text3}
Answer the question with at most 4 words: {demo question3}
Do not provide a Yes/No answer: {demo Answer3}

Read the text: {demo text4}
Answer the question with at most 4 words: {demo question4}
Do not provide a Yes/No answer: {demo Answer4}

Read the text: {query text}
Answer the question with at most 4 words: {query question}
Do not provide a Yes/No answer:
\end{verbatim}

\clearpage

\section{Description of Evaluation Metrics}

\begin{table*}[ht!]
\centering
\renewcommand{\arraystretch}{1.3}
\begin{tabular}{|p{4cm}|p{10cm}|}
\hline
\textbf{Metric} & \textbf{Description} \\
\hline
ANLS (Average Normalized Levenshtein Similarity) & Based on the Levenshtein (edit) distance, which measures the minimum number of single-character edits needed to convert one string into another. More lenient than usual ROUGE metrics and allows partial credit for semantically correct approximate answers. This was a key metric in the NeurIPS 2023 competition that introduced the federated DocVQA dataset~\citep{tobaben2024neurips}. \\
\hline
Exact Match & The fraction of test queries with final LM responses exactly matching the ground truth answer. Considered important in QA tasks as most answers tend to be at most 3 words. \\
\hline
ROUGE-1 & Measures the overlap of unigrams (individual words) between the LM response and the ground truth answer. Counts the number of words in the prediction that also appear in the ground truth. \\
\hline
ROUGE-2 & Measures the overlap of bigrams between the LM response and the ground truth answer. Captures more contextual similarity than ROUGE-1. \\
\hline
ROUGE-L & Captures sentence-level structure similarity by finding the longest sequence of words appearing in both LM response and ground truth in the same order. \\
\hline
BLEU (Bilingual Evaluation Understudy) & Computes the proportion of n-grams (1 to 4) in the LM response that appear in the ground truth. The final score is the geometric mean of n-gram precision multiplied by a penalty term for overly concise answers. \\
\hline
\end{tabular}
\caption{Description of evaluation metrics used in QA tasks.}
\label{tab:qa_metrics}
\end{table*}

\section{Distributions of Number of Tokens for Q\&A tasks}

\begin{figure} [h!]
     \centering
        \includegraphics[width=.45\textwidth]{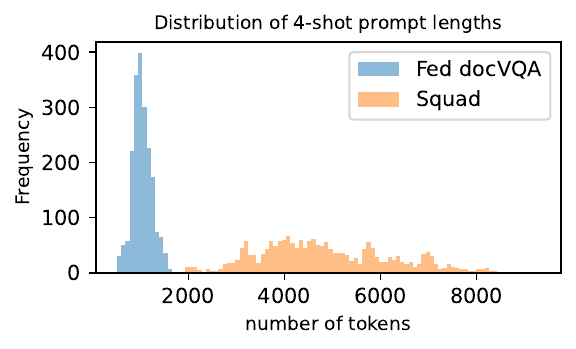}
        \caption{The distributions of number of tokens in the 4 shot prompts (created using demonstration and test examples) when \# shards= 20 for two datasets. The prompts for the fed DocVQA dataset are longer due to verbose nature of the images, hence many more ocr extracted tokens.}
        \label{fig:prompt_sizes}
\end{figure}

\clearpage

\section{Further Experimental Results for Q\&A tasks} \label{sec:further}

\begin{figure} [ht!]
     \centering
        \includegraphics[width=.49\textwidth]{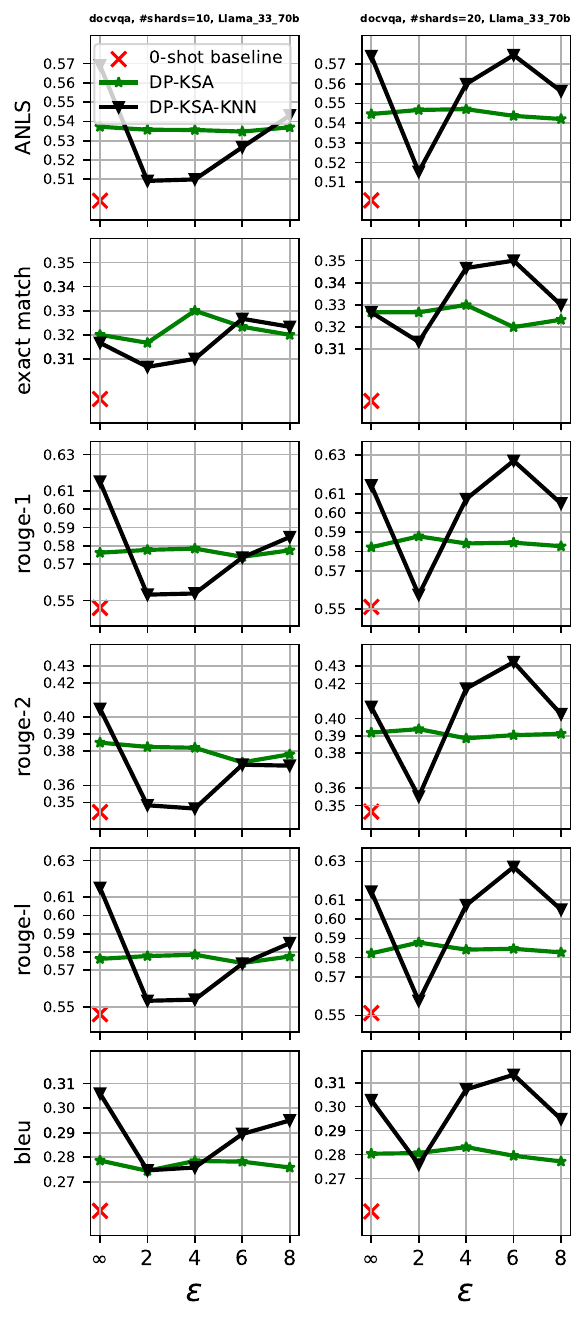}
              \includegraphics[width=.49\textwidth]{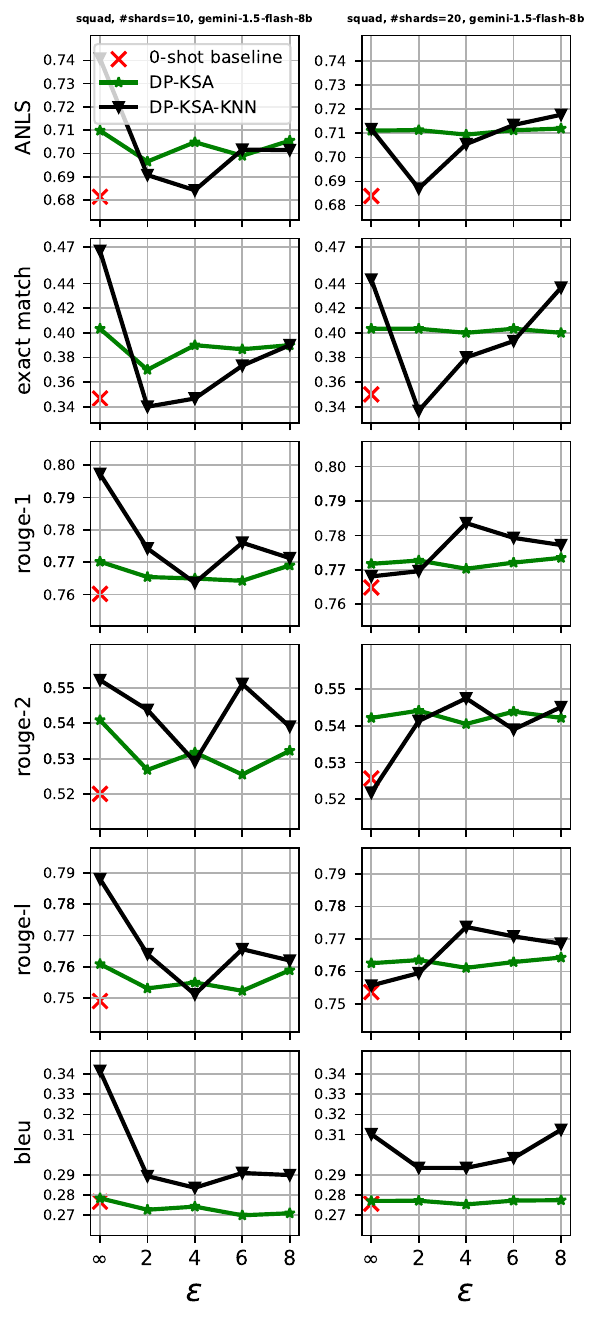}
        \caption{A comparison of DP-KSA and DP-KSA-kNN for average  Q\&A task metrics. Left: DocVQA dataset using Llama 3.3-70B-It.
        Right:  SQuAD dataset using Gemini-1.5-flash-8B.}
        \label{fig:docvqa_llama}
\end{figure}

\end{document}